\documentclass[letterpaper, 10 pt, conference]{ieeeconf}
\IEEEoverridecommandlockouts                              

\usepackage[noformatting,nogeometry]{stdcommands}
\usepackage{graphicx}
\usepackage{dsfont} 
\usepackage{amsthm}
\usepackage{amsmath}
\usepackage{amssymb}
\usepackage{url}
\usepackage[utf8]{inputenc} 
\usepackage[format=hang,singlelinecheck=false,caption=false,font=footnotesize]{subfig}
\usepackage[ruled,linesnumbered]{algorithm2e}

\usepackage{scalerel}
\usepackage[noend]{algpseudocode}
\usepackage[english]{babel}
\usepackage{mathrsfs}
\usepackage{comment}
\usepackage{bm}
\usepackage{hhline}
\newtheorem{proposition}{Proposition}[section]

\title{\LARGE \bf TDR-OBCA: A Reliable Planner for Autonomous Driving \\
in Free-Space Environment}

\begin{document}

\author{%
    \parbox{\linewidth}{\centering
      Runxin He$^{1}$,
      Jinyun Zhou$^{1}$,
      Shu Jiang,
      Yu Wang,
      Jiaming Tao,
      Shiyu Song,
      Jiangtao Hu,
      Jinghao Miao,
      Qi Luo$^{2}$
  }%
  \thanks{$^1$ Authors contributed equally to this paper}
  \thanks{$^2$ Corresponding author}
  \thanks{All Authors are with Baidu USA LLC,
        250 E Caribbean Drive, Sunnyvale, CA 94089
        {\tt\small luoqi06@baidu.com}}
}

\markboth{}{}

\maketitle
\thispagestyle{empty}
\pagestyle{empty}

\begin{abstract}

This paper presents an optimization-based collision avoidance trajectory generation method for autonomous driving in free-space environments, with enhanced robustness, driving comfort and efficiency. Starting from the hybrid optimization-based framework, we introduce two warm start methods, temporal and dual variable warm starts, to improve the efficiency. We also reformulates the problem to improve the robustness and efficiency. We name this new algorithm TDR-OBCA.
With these changes, compared with original hybrid optimization we achieve a $96.67\%$ failure rate decrease with respect to initial conditions, 
$13.53\%$ increase in driving comforts and $3.33\%$ to $44.82\%$ increase in planner efficiency as obstacles number scales.
We validate our results in hundreds of simulation scenarios and hundreds of hours of public road tests in both U.S. and China.
Our source code is available at~\url{https://github.com/ApolloAuto/apollo}.
\end{abstract}
\IEEEpeerreviewmaketitle

\section{Introduction}
%
%
%
%
Collision-free trajectory planning for nonholonomic system is vital to robotic applications, such as autonomous driving~\cite{paden2016survey}~\cite{schwarting2018planning}.
One of its application areas is the autonomous driving in free-space scenarios~\cite{Qi9300224}, where the ego vehicle's starting points or destinations are off-road.
Parking and hailing  belong to these free-space scenarios, because the ego vehicle either goes to or starts from an off-road spot, e.g., parking spot.
In free-space scenarios, an autonomous vehicle may need to drive forward and/or backward via gear shifting to maneuver itself towards its destination.
Trajectories generated by Hybrid A*~\cite{dolgov2008practical} or other sampling/searching based methods~\cite{aoude2010sampling} usually fail to meet the smoothness requirement for autonomous driving.
Thus, an optimization-based method is required to generate feasible trajectories in order to decrease autonomous driving failures~\cite{mayne2000constrained}~\cite{richards2005implementation}.

Though non-holonomic constraints, such as vehicle dynamics, can be well incorporated into Model Predictive Control (MPC) formulations, collision avoidance constraint still remains a challenge due to its non-convexity and non-differentiability~\cite{eele2009path}~\cite{runxin17}.
To simplify the collision constraint, most trajectory planning studies approximate the ego vehicle and obstacles into disks. However, this approximation reduces the problem's geometrical feasible space, which may lead to failures in scenarios with irregularly placed obstacles~\cite{li2015unified}.
If obstacles and the ego vehicle are represented with full dimensional polygons, the collision avoidance MPC is derived into a Mixed-Integer Programming (MIP) problem~\cite{da2019collision}.
Though a variety of numerical algorithms are available to solve MIP such as Dynamical Programming (DP) method~\cite{kellerer2004multidimensional} and the branch and bound method~\cite{floudas1995nonlinear}, due to its completeness, the computation time may fail to satisfy the real-time requirements when the number of nearby obstacles is large~\cite{richards2005mixed}.
To remove the discrete integer variables in the MIP, Zhang~\cite{zhang2018autonomous} presented Hybrid Optimization-based Collision Avoidance (H-OBCA) algorithm.
The distances between the ego vehicle and obstacles are reformulated into their dual expressions, and the trajectory planning problem is transformed into a large-scale MPC problem with only continuous variables.

Inspired by H-OBCA, we present Temporal and Dual warm starts with Reformulated Optimization-based Collision Avoidance (TDR-OBCA) algorithm with improved robustness, driving comfort and efficiency, and integrate it with Apollo Autonomous Driving Platform~\cite{apollo_2019}~\cite{IV9304810}.
Our contributions are:
\begin{enumerate}
    \item \textbf{Robustness}: With two extra warm starts and a reformulation to the cost and constraints in the final MPC, we reduce failure rate by $96.67 \%$, from $37.5 \%$ to $1.25 \%$ with respect to different initial spots in Section~\ref{sec:simulation_simple};
    \item \textbf{Driving Comfort}: With additional smooth penalty terms in the cost function, we show in Section~\ref{sec:simulation_apollo} the steering control outputs have reduced more than $13.53\%$. 
    \item \textbf{Efficiency}: We also show in Section~\ref{sec:simulation_apollo} that TDR-OBCA's solving  time maintains a $3.44\%$ to $81.25\%$ decrease compared with original methods as the number of obstacles increases in different scenarios (including the driving region boundaries, as well as surrounding vehicles and pedestrians), with possible future time reduction if we replace IPOPT~\cite{wachter2006on} with powerful solvers specially designed for MPC, such as GRAMPC~\cite{GRAMPC}. Thus making it applicable to free-space scenarios of different complexities, e.g. parking, pulling over, hailing and even three-point U-turning in the future. 
    \item \textbf{Real Road Tests}: We show in~\ref{sec:road_experiment} that robustness, efficiency and driving comfort of TDR-OBCA is verified on both 282 simulation cases and hundreds of hours of public road tests in both China and USA.
    
\end{enumerate} 
This paper is organized as follows: the problem statement and core algorithms are in Section~\ref{sec:optimization} and~\ref{sec:extension} respectively; the results of both simulations and on-road tests are presented in Section~\ref{sec:experiments}.

\section{Problems statement}
\label{sec:optimization}
TRD-OBCA aims to generate a smooth collision-free trajectory for autonomous vehicles in free spaces, i.e., parking lots or off-road regions.
We formulate the collision-free and smoothness requirements as constrains of a MPC optimization problem, which is similar to H-OBCA algorithm~\cite{zhang2018autonomous}. 

At time $k$, the autonomous ego vehicle's state vector is $x\pb{k} = \sqparen{x_{x}\pb{k}, x_{y}\pb{k}, x_v\pb{k}, x_{\phi}\pb{k}}^T\in \R^4$,
where $x_x\pb{k}$ and $x_y\pb{k}$ is vehicle's latitude and longitude position, $x_v\pb{k}$ is the vehicle's velocity and $x_{\phi}\pb{k}$ is heading (yaw) angle in radius. The control command from steering and brake/throttle is formulated as $u\pb{k} = \sqparen{\delta\pb{k}, a\pb{k} }^T \in \R^2$, where $\delta\pb{k}$ is the steering and $a\pb{k}$ is the acceleration. The ego vehicle's dynamic system is modeled by kinematic bicycle model, defined as Eq.\ref{eq:dynamic_model},
\begin{equation}
\label{eq:dynamic_model}
    x\pb{k} = f\pb{x\pb{k-1}, u\pb{k-1}},
\end{equation}
and detailed as Eq.\ref{eq:detailed_dynamic_model}, 
\begin{equation}
\label{eq:detailed_dynamic_model}
\begin{aligned}
    & x_x\pb{k+1} = x_v\pb{k}\cos{(x_{\phi}\pb{k})}d_t + x_x\pb{k}, \\
    & x_y\pb{k+1} = x_v\pb{k}\sin{(x_{\phi}\pb{k})}d_t + x_y\pb{k}, \\
    & x_{\phi}\pb{k+1} = x_v\pb{k}\frac{\tan{(\delta\pb{k}})}{L}d_t +  x_{\phi}\pb{k}, \\
    & x_v\pb{k+1} = a\pb{k}d_t + x_v\pb{k}, \\
\end{aligned}
\end{equation}
where $L$ is the wheelbase length and $d_t$ is the discretization time step,

Compared to H-OBCA, we have two major changes to the problem's formulation. The first is, instead of using variant time steps, the time horizon~$\sqparen{0,T}$ is evenly discretized into $K$ steps. The time horizon $T$ is derived based on the estimated trajectory distance and vehicle dynamic feasibility. The detailed implementation is introduced in Section~\ref{subsec:temporal_warm_start}. 
The other major change is that we reformulate the constraints and cost function to reduce control efforts and increase trajectory smoothness, which is introduced in Section~\ref{subsec:relaxation}.
Here we first present the formulation to H-OBCA, which transforms the collision avoidance problem form a MIP to a continous MPC with nonlinear constrains,
\begin{equation}
  \label{eq:original}
  \begin{aligned}
    &\min_{\substack{%
        \bm{x, u, \mu, \lambda} \\
      }}\,
    \sum_{k=1}^{K} l~\pb{x\pb{k}, u\pb{k-1}} \\
    &\text{subject to:} \\
    & x\pb{0} = x_0, \\
    & x\pb{K} = x_{F}, \\
    & x\pb{k} = f\pb{x\pb{k-1}, u\pb{k-1}}, \\
    & h\pb{x\pb{k}, u\pb{k}} \leq 0, \\
    &-g^T \mu_m{\pb{k}}
        + \pb{A_{m} t\pb{x\pb{k}} - b_{m}}^T \lambda_m{\pb{k}} > d_{min}, \\
    & G^T \mu_m{\pb{k}}
        + R\pb{x\pb{k}}^T {A_{m}}^T \lambda_m{\pb{k}}
        = 0, \\
    &\norm{{A_{{m}}}^T \lambda_m{\pb{k}}}_{2} \leq 1, 
    \lambda_{m} \pb{k} \succeq_{\kappa} 0,
     \mu_m {\pb{k}} \succeq_{\kappa} 0, \\
    &\text{for } k = 1, \dots K,
        \text{ } m = 1, \dots, M.
  \end{aligned}
\end{equation}
Where $l$ is the cost function, its detailed formulation and our corresponding reformulation are in Section~\ref{subsec:relaxation}.
The optimization is performed over $\bm{x} =\sqparen{x\pb{0},\ldots,x\pb{K}}$, $\bm{u} =\sqparen{u\pb{0},\ldots,u\pb{K-1}}$ and dual variables to the dual formulas of distances between ego vehicle and obstacles~\cite{zhang2018autonomous}, $\bm{\mu} =\sqparen{\mu_{1}\pb{1},\ldots,\mu_{1}\pb{K},\mu_{2}\pb{1}\ldots,\mu_{M}\pb{K}}$ 
and $\bm{\lambda} =\sqparen{\lambda_{1}\pb{1},\ldots,\lambda_{1}\pb{K},\lambda_{2}\pb{1}\ldots,\lambda_{M}\pb{K}}$.
$x_0$ and $x_F$ are ego vehicle's start and end state respectively.
Constraints $h\pb{x\pb{k}, u\pb{k}}$ express the categories of vehicle dynamic limitations and driving comforts. In this paper, we apply box constraints to describe a subset of the union of dynamic limits and driving comforts w.r.t. steering angle and acceleration.
The vehicle geometry is approximated as a polygon with line segments, described by matrix $G$ and vector $g$.
Obstacles are described as combinations of polygons as well, the $m$-th polygon is represented as a matrix $A_m$ and the vector $b_m$.
A complicated obstacle may be expressed by several polygons with more line segments than a simple one.
At time step $k$, the set distance between the ego vehicle and $m$-th obstacle polygon should be larger than $d_{min}$ for safety.
Detailed descriptions to expressions of obstacles' geometry and safety distance are referred to H-OBCA~\cite{zhang2018autonomous}.
In production, in order to keep autonomous vehicle's safety, when the H-OBCA problem or our TDR-OBCA problem are failed due to unexpected scenarios, for example obstacles' aggressive maneuvers, the redundant safety module of the driving platform will generate a fallback trajectory instead.
The details of the fallback trajectory is referred to the platform's source code and not discussed in this paper.

\section{Core algorithms}
\label{sec:extension}
Due to its high non-convexity and non-linearity, solving problem in Eq.~\eqref{eq:original} requires heavy computation efforts. Even in simulation, it may take several minutes to generate a trajectory. We introduce the temporal profile and dual variable warm starts together with hybrid A* to speed up the convergence of the final optimization problem. We also reformulate the MPC problem with new penalty terms to guarantee the smoothness of the trajectory and soft constraints to end state spot for robustness.
The architecture of TDR-OBCA is shown in Fig.~\ref{figure:open_space_architect_4} and the highlighted blocks are the core algorithms, which are introduced in this section. 

\begin{figure}[!htb]
\centering
\includegraphics[width=0.7\linewidth]{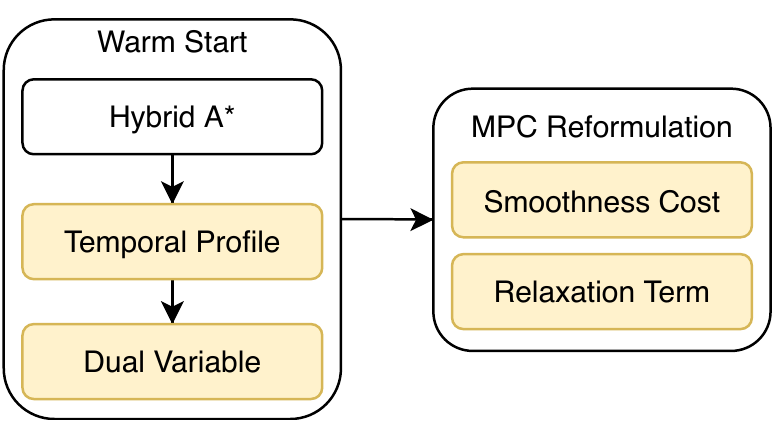}
\caption{TDR-OBCA Architecture.}
\label{figure:open_space_architect_4}
\end{figure}

\subsection{Temporal profile warm start}
\label{subsec:temporal_warm_start}
Hybrid A* is used to heuristically generate $x_{x}$, $x_{y}$ and $x_{phi}$ as part of a collision free trajectory to problem~\eqref{eq:original}.
Then, a temporal profile method estimates $x_v$ and $x_{\phi}$ based on the path result from Hybrid A*.
However, the temporal profile results generated by directly differentiating path points from Hybrid A* is not smooth enough for control layer.
In TDR-OBCA, an optimization-based temporal profile method is introduced to improve the dynamic feasibility and smoothness of the trajectory, thus reducing the failure rate.

The detail of the temporal profile method in TDR-OBCA has been discussed in another work of one of the authors~\cite{yajia}.
In this process, the Hybrid A* output path is first partitioned by different vehicle gear locations, i.e., forward gear or reverse gear. Then, on each path, we optimize the state space of longitudinal traverse distance and its derivatives with respect to time.
In our application, the total time horizon is derived based on the ego vehicle's minimal traverse time and its dynamic feasibility.
Given the ego vehicle's maximum acceleration $a_{max}$, maximum speed $v_{max}$, and maximum deceleration $-a_{max}$ on a traverse distance of $s$, total horizon is chosen as $T = r \frac{v_{max}^2 + s a_{max}}{a_{max}  v_{max}}$ where $r \in [1.2, 1.5]$ is a heuristic ratio.
 
The modified optimization problem for temporal profile is a quadratic programming (QP) problem, which can be efficiently solved by high performance numerical solvers~\cite{boyd2011distributed}.

\subsection{Dual variable warm start}
\label{subsec:dual_variable}
In TDR-OBCA, we use dual variable initialization to provide better initial values to $\bm{\mu}$ and $\bm{\lambda}$ for the final MPC problem, such as~\eqref{eq:original}, to achieve a fast convergence to the optimal value.
Because the cost function in problem~\eqref{eq:original} does not contain $\bm{\mu}$ and $\bm{\lambda}$, its direct optimization-based warm start problem is not well defined.
To address this issue, we first introduce a slack variable, $\bm{d}$, to the dual variable warm start problem. We define $\bm{d} =\sqparen{d_{1}\pb{1},\ldots,d_{1}\pb{K},d_{2}\pb{1}\ldots,d_{M}\pb{K}} $, where $d_m\pb{k}$ represents the negative value of a safety distance between the ego vehicle and the $m$-th obstacle polygon at time $k$. Thus, the dual warm up problem is reformulated as
\begin{equation}
  \label{eq:original_dual}
  \begin{aligned}
    &\min_{\bm{\mu}, \bm{\lambda}, \bm{d}}\, \sum_{m=1}^{M} \sum_{k=1}^{K} d_m\pb{k} \\
    &\text{subject to:} \\
    &
    -g^T \mu_m{\pb{k}}
    + \pb{A_{m} t\pb{\Tilde{x}^*\pb{k}} - b_{m}}^T \lambda_m{\pb{k}}\\
    & \hspace{4.0em}
    + d_m\pb{k} = 0, \\
    &
    G^T \mu_m{\pb{k}}
        + R\pb{\Tilde{x}^*\pb{k}}^T {A_{m}}^T \lambda_m{\pb{k}}
        = 0, \\
    &
    \norm{{A_{{m}}}^T \lambda_m{\pb{k}}}_{2} \leq 1, \\
    &
    \lambda_{m} \pb{k} \succeq 0,
    \mu_m {\pb{k}} \succeq 0,
    d_m\pb{k} < 0, \\
    &\hspace{4.0em}
    \text{for } k = 1, \dots K,
        \text{ } m = 1, \dots, M,
  \end{aligned}
\end{equation}
where $\Tilde{x}^*$ is the vehicle states from Hybrid A* .
To improve the pipeline's computation efficiency, we transform the problem~\eqref{eq:original_dual} from a convex Quadratic Constrained Quadratic Programming (QCQP) into a simple QP problem as Eq.~\eqref{eq:relax_dual}.
The transformed QP problem~\eqref{eq:relax_dual} can be easily solved by distributed and parallel computing algorithms, such as Operator Splitting and ADMM~\cite{boyd2011distributed}.
\begin{equation}
  \label{eq:relax_dual}
  \begin{aligned}
    & \min_{\bm{\mu}, \bm{\lambda}, \bm{d}}\,  
      \frac{1}{\beta} \sum_{m=1}^M \norm{{A_{{m}}}^T \lambda_m{\pb{k}}}_{2}^2
      + \sum_{m=1}^{M} \sum_{k=1}^{K} d_m\pb{k}\\
    & \text{subject to:} \\
    & 
    -g^T \mu_m{\pb{k}}
    + \pb{A_{m} t\pb{\Tilde{x}^*\pb{k}} - b_{m}}^T \lambda_m{\pb{k}} \\
    & \hspace{4.0em}
    + d_m\pb{k} = 0, \\
    & 
    G^T \mu_m{\pb{k}}
        + R\pb{\Tilde{x}^*\pb{k}}^T {A_{m}}^T \lambda_m{\pb{k}}
        = 0, \\
    & 
    \lambda_{m} \pb{k} \succeq 0,
    \mu_m {\pb{k}} \succeq 0,
    d_m\pb{k} < 0, \\
    & \hspace{4.0em}
    \text{for } k = 1, \dots K,
        \text{ } m = 1, \dots, M.
  \end{aligned}
\end{equation}

Now we prove that the optimal solution to the relaxed problem~\eqref{eq:relax_dual} leads to an sub-optimal point to problem~\eqref{eq:original_dual} and show one corresponding transformation formula.

\begin{proposition}
\label{thm:dual1}
The feasible set of optimization problem~\eqref{eq:original_dual} is a subset of~\eqref{eq:relax_dual}.
Thus, given
 $\{\tilde{\bm{\mu}}, \tilde{\bm{\lambda}}, \tilde{\bm{d}}\}$ as the optimum solution to~\eqref{eq:original_dual},
$\{\bm{\mu}^{o}, \bm{\lambda}^{o}, \bm{d}^{o}\}$ as the optimum solution to~\eqref{eq:relax_dual}, then
\begin{equation}
  \label{eq:dual_bnd}
  \begin{aligned}
    & \frac{1}{\beta} \sum_{m=1}^M \norm{{A_{{m}}}^T \lambda^o_m{\pb{k}}}_{2}^2
    + \sum_{m=1}^{M} \sum_{k=1}^{K} d^o_m\pb{k} \leq \\
    & \frac{1}{\beta} \sum_{m=1}^M \norm{{A_{{m}}}^T \tilde{\lambda}_m{\pb{k}}}_{2}^2
    + \sum_{m=1}^{M} \sum_{k=1}^{K} \tilde{d}_m\pb{k}.
  \end{aligned}
\end{equation}
Moreover, the optimization solution to problem~\eqref{eq:relax_dual} leads to a feasible point to optimization problem~\eqref{eq:original_dual}.
\end{proposition}
\begin{proof}
Since optimization problem in~\eqref{eq:original_dual} has one more type of constraints than problem~\eqref{eq:relax_dual}, it is trivial to show $\{\tilde{\bm{\mu}}, \tilde{\bm{\lambda}}, \tilde{\bm{d}}\}$ also satisfies the constraints to problem~\eqref{eq:relax_dual}, thus inequality~\eqref{eq:dual_bnd} satisfies as well.

We prove the left statement by constructing such a transformation. 
Assume $\{\bm{\mu}^{o}, \bm{\lambda}^{o}, \bm{d}^{o}\}$ is the optimal solution to~\eqref{eq:relax_dual},
construct a new combination of dual variables, $\{\bm{\mu}^{m}, \bm{\lambda}^{m}, \bm{d}^{m}\}$, as
for $k = 1, \dots K$, and $m = 1, \dots, M$,
if $\norm{{A_{{m}}}^T \lambda^o_m{\pb{k}}}_{2} > 1$,
$\lambda^m_{m}\pb{k} = \frac{1}{\norm{{A_{{m}}}^T \lambda^o_m{\pb{k}}}} \lambda^o_{m}\pb{k}$,
$\mu^m_{m}\pb{k} = \frac{1}{\norm{{A_{{m}}}^T \lambda^o_m{\pb{k}}}} \mu^o_{m}\pb{k}$,
$d^m_{m}\pb{k} = \frac{1}{\norm{{A_{{m}}}^T \lambda^o_m{\pb{k}}}} d^o_{m}\pb{k}$;
otherwise, 
$\lambda^m_{m}\pb{k} = \lambda^o_{m}\pb{k}$,
$\mu^m_{m}\pb{k} = \mu^o_{m}\pb{k}$,
$d^m_{m}\pb{k} = d^o_{m}\pb{k}$.

By the scaling transfer above, we keep $\norm{{A_{{m}}}^T \lambda^m_m{\pb{k}}}_{2}$ smaller than $1$ and satisfying all the other constraints in~\eqref{eq:relax_dual}, thus $\{\bm{\mu}^{m}, \bm{\lambda}^{m}, \bm{d}^{m}\}$ satisfies the constraints in~\eqref{eq:original_dual}.
\end{proof}

Moreover, since the above scaling is strictly larger than $1$, it is able to show the further relation between optimum solutions to problem~\eqref{eq:original_dual} and problem~\eqref{eq:relax_dual} as following,
\begin{proposition}
\label{thm:dual_bnd2}
Given
 $\{\tilde{\bm{\mu}}, \tilde{\bm{\lambda}}, \tilde{\bm{d}}\}$ as the optimum solution to~\eqref{eq:original_dual},
$\{\bm{\mu}^{m}, \bm{\lambda}^{m}, \bm{d}^{m}\}$ as the transfer constructed in Proposition~\ref{thm:dual1} to the optimum solution to~\eqref{eq:relax_dual},
\begin{equation*}
  \begin{aligned}
    & \sum_{m=1}^{M} \sum_{k=1}^{K} \tilde{d}_m\pb{k} \leq 
      \sum_{m=1}^{M} \sum_{k=1}^{K} d^m_m\pb{k} \leq \\
    & \hspace{1.0em} \frac{1}{\beta} \sum_{m=1}^M \norm{{A_{{m}}}^T \lambda^m_m{\pb{k}}}_{2}^2
    + \sum_{m=1}^{M} \sum_{k=1}^{K} d^m_m\pb{k}.
  \end{aligned}
\end{equation*}
\end{proposition}

Propositions~\ref{thm:dual1} and~\ref{thm:dual_bnd2} show the relation between the dual warm up QCQP~\eqref{eq:original_dual} and QP~\eqref{eq:relax_dual}.
From the aspect of numerical computation, we apply the QP problem~\eqref{eq:relax_dual} for dual warm up since it is more efficient to solve. 

\subsection{MPC problem reformulation}
\label{subsec:relaxation}
To increase trajectory smoothness and reduce control efforts, we first define the cost term, $l(\cdot)$, in problem~\eqref{eq:original} as,
\begin{equation}
  \label{eq:cost_func}
  \begin{aligned}
    &\alpha_{x}\norm{x\pb{k}}_2^2 + \alpha_{x'}\norm{x\pb{k} - x\pb{k-1}}^2_2 \\
    & \hspace{1.0em} +\alpha_{u} \norm{u\pb{k-1}}_2^2 + \alpha_{\tilde{u}} \norm{u\pb{k-1} - \tilde{u}\pb{k-1}}_2^2,
   \end{aligned}
\end{equation}
where $\alpha_{x}$, $\alpha_{x'}$, $\alpha_{u}$ and $\alpha_{\tilde{u}}$ are corresponding hyper-parameters.
In Eq.~\eqref{eq:cost_func}, the first two terms measure the trajectory's first and second order smoothness respectively, the third term measures control energy usage, and the fourth term measures differences between two consecutive control decisions, where $\tilde{u}$ represents the previous control decision at the corresponding time step.
Note that the fourth term is crucial to road tests where the real vehicle is involved. The control actuators, i.e., steering, braking and throttling, have limited reaction speed. Large fluctuations between consecutive control decisions may make the actuator fail to track the desired control commands.

This basic cost function in Eq.~\eqref{eq:original} is further modified in TDR-OBCA to overcome two major practical issues: 1) $d_{min}$ in equation is a hyper-parameter for minimum safety distance, but it is hard to tune for general scenarios; 
2) both initial and end state constraints make the nonlinear optimization solver slow and, in extreme case, hard to find a feasible solution.

To address these issues, we introduce two major modifications to the MPC problem~\eqref{eq:original}.
First, we introduce a collection of slack variables, $\bm{d}$, which is defined in section~\ref{subsec:dual_variable} to the MPC problem; second the end state constraints are relaxed to soft constraints inside the cost function.
The new cost function is
\begin{equation}
\begin{aligned}
&\mathcal{J}~\pb{\bm{x}, \bm{u}, \bm{d}} = \sum_{k=1}^{K} l~\pb{x\pb{k}, u\pb{k-1}} \\
& \hspace{1.0em} + \alpha_e \norm{x\pb{K} - x_F}_2^2 + \beta \sum_{m=1}^{M} \sum_{k=1}^{K} d_m\pb{k},
\end{aligned}
\end{equation}
where $\alpha_e > 0$ is the hyper-parameter to minimize the final state's position to the target, and $\beta > 0$ is the hyper-parameter to those cost terms which maximize the total safety distances between vehicle and obstacles.
The reformulated MPC problem with slack variables is formulated as,

\begin{equation}
  \label{eq:relax}
  \begin{aligned}
    & \min_{
        \bm{x},~\bm{u},~\bm{d},~\bm{\mu},~\bm{\lambda}
      }
    \mathcal{J}~\pb{\bm{x}, \bm{u}, \bm{d}} \\
    & \text{subject to:} \\
    & \hspace{1.0em} \quad x\pb{0} = x_0,\\
    & \hspace{2.0em} 
    x\pb{k+1} = f\pb{x\pb{k}, u\pb{k}}, \\
    & \hspace{2.0em}
    h\pb{x\pb{k}, u\pb{k}} \leq 0, \\
    & \hspace{2.0em}
    -g^T \mu_m{\pb{k}}
        + \pb{A_{m} t\pb{x\pb{k}} - b_{m}}^T \lambda_m{\pb{k}}
        \\
    & \hspace{4.0em}
       + d_{m}\pb{k} = 0, \\
    & \hspace{2.0em} 
    G^T \mu_m{\pb{k}}
        + R\pb{x\pb{k}}^T {A_{m}}^T \lambda_m{\pb{k}}
        = 0, \\
    & \hspace{2.0em}
    \norm{{A_{{m}}}^T \lambda_m{\pb{k}}}_{2} \leq 1, \\
    & \hspace{2.0em}
    \lambda_{m} \pb{k} \succeq 0,
    \mu_m {\pb{k}} \succeq 0,
    d_m\pb{k} < 0, \\
    & \hspace{4.0em}
    \text{for } k = 1, \dots K,
        \text{ } m = 1, \dots, M,
  \end{aligned}
\end{equation}
where the notations to $d_m\pb{k}$ and $\bm{d}$ follow the problem~\eqref{eq:original_dual}.

\section{Experiment results}
\label{sec:experiments}
TDR-OBCA algorithm is validated in both simulations and real world road tests.
In subsection~\ref{sec:simulation}, we start showing TDR-OBCA's robustness at different starting positions.
Then, we further verify TDR-OBCA's robustness and efficiency in end-to-end scenario-based simulations. 
In subsection~\ref{sec:road_experiment}, TDR-OBCA is deployed on real autonomous vehicles, which proves the trajectories provided by TDR-OBCA have control-level smoothness as well as its efficiency and robustness in real world.

\subsection{Simulations}
\label{sec:simulation}
In this subsection, we present two categories of simulations.
\subsubsection{Robustness to various starting positions}
\label{sec:simulation_simple}
We design a no-obstacle valet parking scenario with regular curb boundaries to show the robustness of TDR-OBCA over its competitive algorithms.
The simulation setups and parameters are shown in Fig.~\ref{figure:simulation_scenario}.
\begin{figure}[!h]
 \centering
\includegraphics[width=1.0\linewidth]{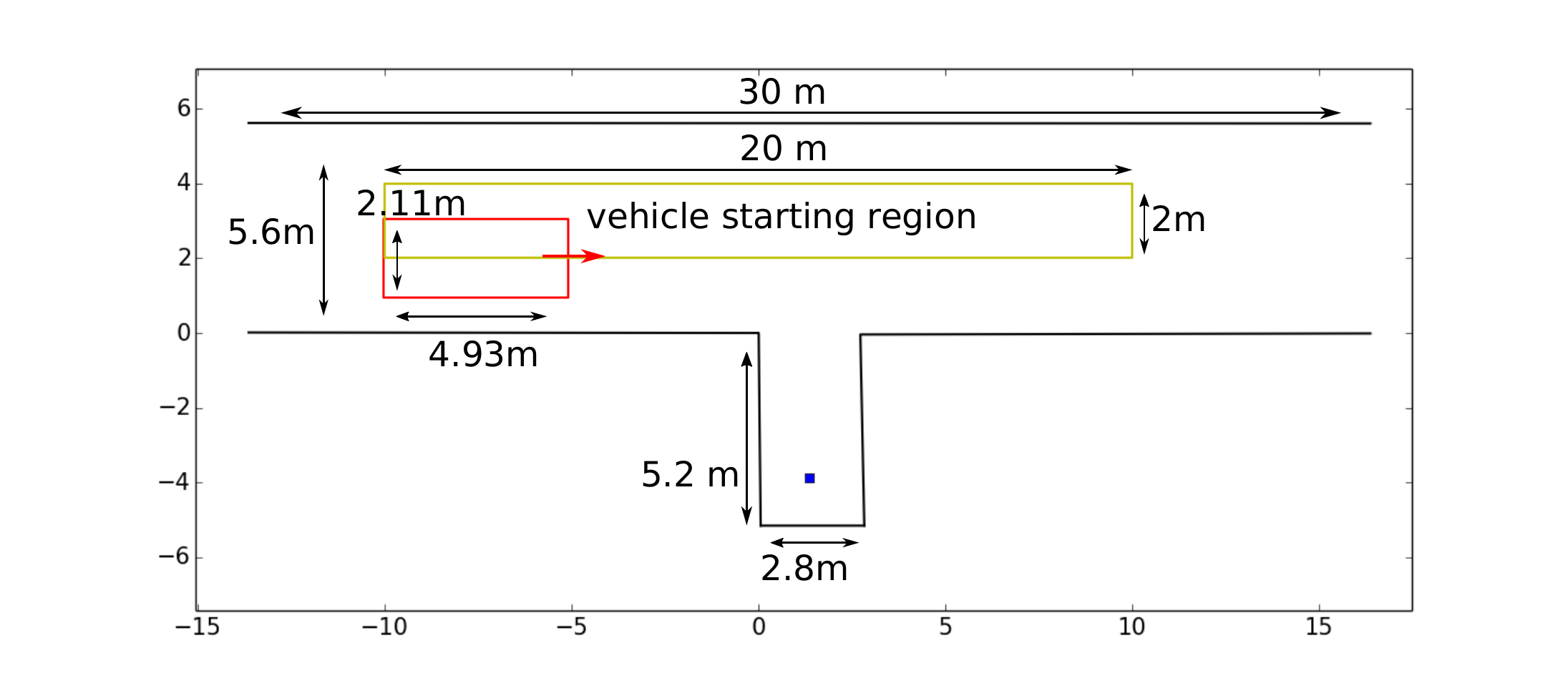}
 \caption{Robustness experiment settings. Ego vehicle's starting positions are within the yellow box. The red box is the ego vehicle at one starting position. The vehicle is $1769$~kg in weight with a $2.8$ $m$ wheelbase. Its steering range is $[-0.5, 0.5]$ $\text{rad}$ and steering rate is $[-0.5, 0.5]$ $\text{rad}/s$. The acceleration range is set to $[-1, 1]$ $m/s^2$ and the speed range is $[-1, 2]$ $m/s$.}
 \label{figure:simulation_scenario}
\end{figure}

The simulation includes 80 cases identified by ego vehicle's different starting positions.
The starting positions are evenly sampled on a grid of $x \in [-10, 10]\ m$ with interval $1.0\ m$ and $y \in [2, 4]\ m$ with interval $0.5\ m$. 
The starting heading angle is set to be zero, facing to the right hand side in Fig.~\ref{figure:simulation_scenario}. The ending parking spot remains the same for all cases. 
The optimization problems are implemented and simulated with a i7 processor clocked at 2.6 GHz.

\begin{table}[!h]
\centering
\caption{Robustness statistics for cases in Fig.~\ref{figure:simulation_scenario}. The failure rate is calculated by the number of cases when the algorithm fails to solve the problem divided by the total case number.}
\normalsize
\setlength{\tabcolsep}{0.5em}
\begin{tabular}{p{2.5cm} |p{2.5cm} |p{2.5cm} }
 \hline
Algorithm & Failure Rate & \textbf{Reduced Rate}\\
 \hhline{=:=:=}
H-OBCA & $37.5\%$ & \textbf{$\diagdown$} \\
 \hline
TD-OBCA & $12.5\%$ & \textbf{66.67\%}\\
 \hline
TDR-OBCA & $1.25\%$ & \textbf{96.67\%}\\
 \hline
 
\end{tabular}
\label{table:simulations}
\end{table}
To show robustness, in Table~\ref{table:simulations} we compare failure rates of three algorithms.
They are: 1) H-OBCA as the benchmark; ii) TD-OBAC algorithm with TDR-OBCA's warm starts proposed in Section~\ref{subsec:temporal_warm_start} and~\ref{subsec:dual_variable}; iii) complete TDR-OBCA algorithm.
The failure rate drops from the benchmark's $37.5 \%$ to $1.25 \%$ by applying TDR-OBCA, whereas the rest of  failures are due to violation of vehicle dynamics.

Besides robustness, we also compare the optimal control output from H-OBCA and TDR-OBCA.
Fig.\ref{figure:smoothness} shows the steering output from H-OBCA and TDR-OBCA from one of the 80 cases as a typical example.
The y-axis is the steering angle in rad. 
The steering angle output from TDR-OBCA (the red line) has less sharp turns compared to that from H-OBCA (the blue line). Similarly, in the plots to acceleration and jerk, which is the rate of change of acceleration with time, TDR-OBCA also shows more smoothness compared to the trajectory provided by H-OBCA.

As a supplement to Fig.~\ref{figure:smoothness}, we compare the minimum, maximum and mean control outputs along with jerks of the two algorithms in Table ~\ref{table:simulation_smoothness}. On average, TDR-OBCA reduced $13.53\%$ steering output, $1.42\%$ acceleration and $3.34\%$ jerk compared with H-OBCA. That means the steering commands generated by TDR-OBCA are smoother, the accelerations and jerks for TDR-OBCA are a little better than H-OBCA.

\subsubsection{End-to-end scenario-based simulations}
\label{sec:simulation_apollo}
In this subsection, we demonstrate TDR-OBCA performance in end-to-end simulations.


\begin{figure}[!htb]
\begin{minipage}{0.49\textwidth}
 \centering
 \subfloat[Steering comparison.\label{figure:steering_smooth}]{\includegraphics[width=0.95\linewidth]{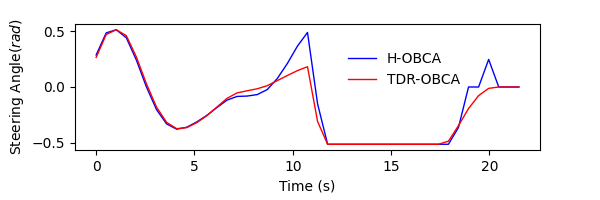}}
\end{minipage}\hfill
\begin{minipage}{0.49\textwidth}
 \centering
 \subfloat[Acceleration comparison\label{figure:acc_smooth}]{\includegraphics[width=0.95\linewidth]{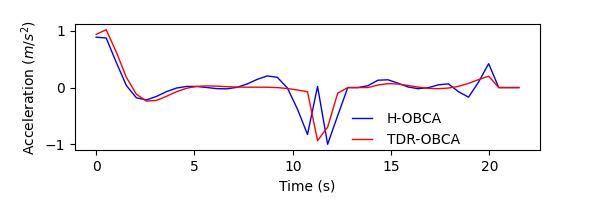}} 
\end{minipage}\hfill
\begin{minipage}{0.49\textwidth}
 \centering
 \subfloat[Jerk comparison\label{figure:jerk_smooth}]{\includegraphics[width=0.95\linewidth]{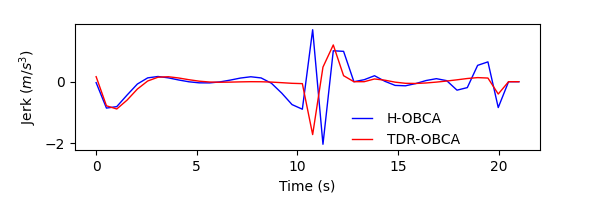}} 
\end{minipage}\hfill
\caption{Control outputs to H-OBCA and TDR-OBCA.}
\label{figure:smoothness}
\end{figure} 



\begin{table}[h!]
\centering
\caption{Trajectory evaluation results over the successful cases.}
\setlength{\tabcolsep}{0.5em}
\begin{tabular}{p{1.7cm}|p{1.3cm} |p{1.6cm} |p{1.6cm}}
 \hhline{--|-|-}
 &  &\textbf{H-OBCA} & \textbf{TDR-OBCA}\\
 \hhline{--|-|-}
   Steering& Mean &$0.2048$ & \textbf{0.1771} \\
 \hhline{~-|-|-}
   Angle& Max &$3.2762$ & $3.2762$ \\
  \hhline{~-|-|-}
   ($rad$)& Min &$-3.8644$ & $-3.8644$ \\
        \hhline{~-|-|-}
   & Std Dev & $0.4317$ & $0.4285$ \\
   \hhline{==|=|=}
 & Mean &$0.3392$ & \textbf{0.3344} \\
 \hhline{~-|-|-}
    Acceleration& Max &$5.3351$ & $5.3351$ \\
  \hhline{~-|-|-}
   ($m/s^{2}$)& Min &$-0.5127$ & $-0.5127$ \\
     \hhline{~-|-|-}
   & Std Dev & $0.4074$ & $0.4215$ \\
   \hhline{==|=|=}
    & Mean &$ 0.2663$ & \textbf{0.2574} \\
 \hhline{~-|-|-}
    Jerk& Max &$10.0583$ & $10.6021$ \\
  \hhline{~-|-|-}
   ($m/s^{3}$)& Min & $-10.2744$ & $-10.2744$ \\
  \hhline{~-|-|-}
   & Std Dev & $0.6215$ & $0.6731$ \\
   \hhline{--|-|-}
 \hline
\end{tabular}
\label{table:simulation_smoothness}
\end{table}

We introduce different types of obstacles and application scenarios (i.e., starting and ending positions). 
Among hundreds of cases presented at~\url{https://bce.apollo.auto}, we chose three typical scenario types, which are valet parking, parallel parking and hailing.
TDR-OBCA successfully generates trajectories for all these scenarios as shown in Fig.~\ref{figure:simulation_reuslts}, which further verifies its robustness.
With different environments, initial and destination spots, TDR-OBCA is able to provide collision free trajectories for most of the cases.

\begin{figure}[!htb]
\begin{minipage}{0.16\textwidth}
 \centering
 \subfloat[valet parking\label{figure:valet_parking_simulation}]{\includegraphics[width=.9\linewidth]{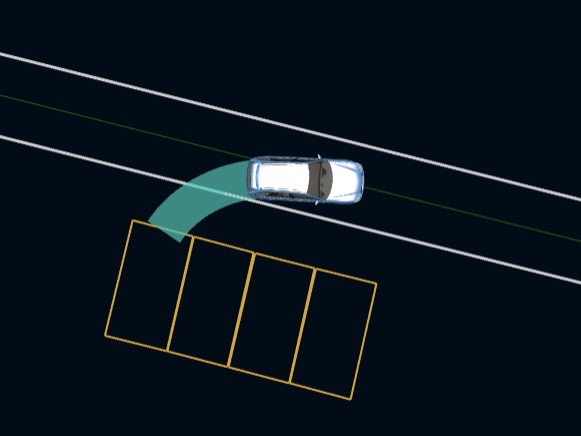}}
\end{minipage}\hfill
\begin {minipage}{0.16\textwidth}
 \centering
 \subfloat[pull over\label{figure:parallel_parking_simulation}]{\includegraphics[width=.9\linewidth]{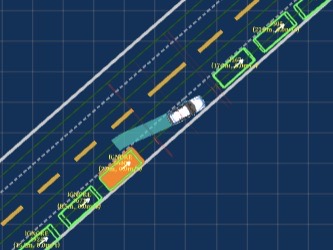}} 
\end{minipage}\hfill
\begin{minipage}{0.16\textwidth}
 \centering
 \subfloat[hailing\label{figure:hailing_simulation}]{\includegraphics[width=.9\linewidth]{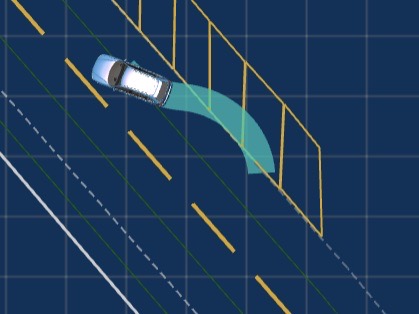}}
\end{minipage}\hfill
\begin{minipage}{0.16\textwidth}
 \centering
 \subfloat[valet parking \newline trajectory\label{figure:valet_parking_zigzag}]{\includegraphics[width=.9\linewidth]{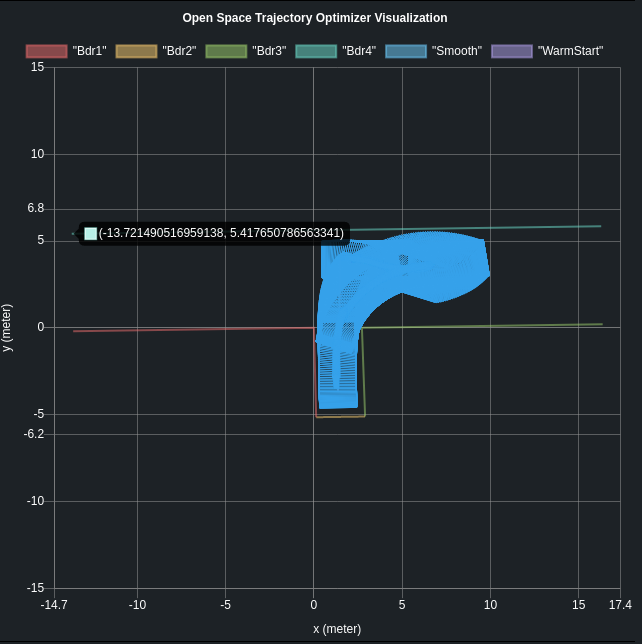}}
\end{minipage}\hfill
\begin{minipage}{0.16\textwidth}
 \centering
 \subfloat[pull over \newline trajectory\label{figure:parallel_parking_zigzag}]{\includegraphics[width=.9\linewidth]{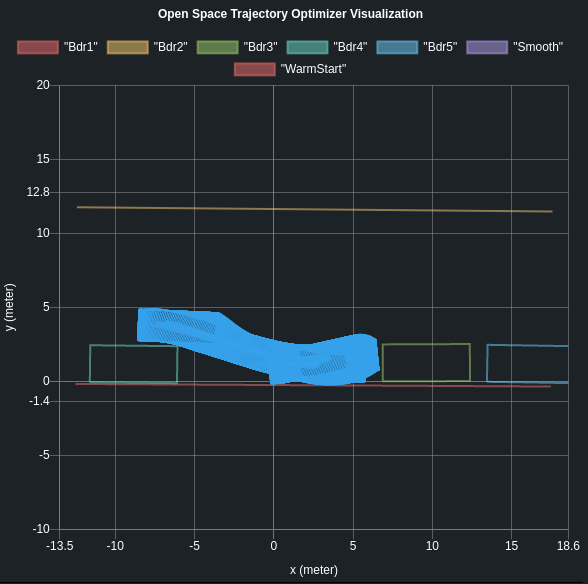}}
\end{minipage}\hfill
\begin{minipage}{0.16\textwidth}
 \centering
 \subfloat[hailing \newline trajectory \label{figure:hailing_zigzag}]{\includegraphics[width=.9\linewidth]{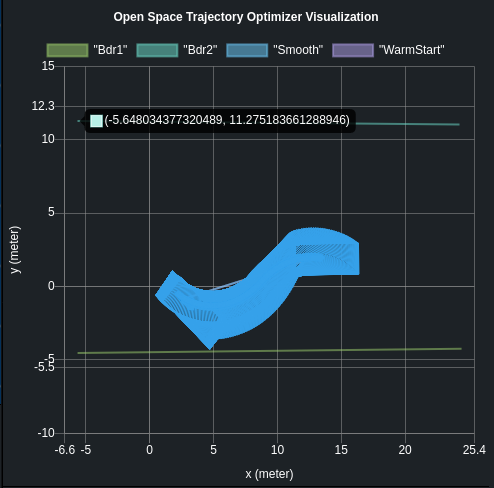}}
\end{minipage}\hfill
\caption{Planning trajectories in simulation including scenarios: 1) valet parking:~\ref{figure:valet_parking_simulation},~\ref{figure:valet_parking_zigzag}; 2) pull over:~\ref{figure:parallel_parking_simulation},~\ref{figure:parallel_parking_zigzag}; 3) hailing:~\ref{figure:hailing_simulation},~\ref{figure:parallel_parking_zigzag}.}
\label{figure:simulation_reuslts}
\end{figure}

For valet parking scenario cases in Fig.~\ref{figure:simulation_reuslts}, their settings are identified by different curb shapes and surrounding obstacles as shown in Fig.~\ref{figure:complex_simulation_reuslts}.
We classify obstacles to two types. Type A obstacles are road boundaries, and Type B obstacles are vehicles or pedestrians.
H-OBCA and TDR-OBCA's computation time and the line segment numbers of both types of obstacles in each simulation are listed in Table~\ref{table:valet_parking_complex_simulation_data}. 
Based on this table, when the simulation becomes more complicated from case ($a$) to ($e$), the total time cost by TDR-OBCA, $t_{t,TDR}$, is not increased as fast as H-OBCA, $t_{t,H}$, which validates TDR-OBCA's computation efficiency.


\begin{figure*}[!htb]
\begin{minipage}{0.19\textwidth}
 \centering
 \subfloat[Pedestrian Obstacle\label{figure:pedestrain_obstacle}]{\includegraphics[trim=0 0 0 0, clip, width=1.0\linewidth]{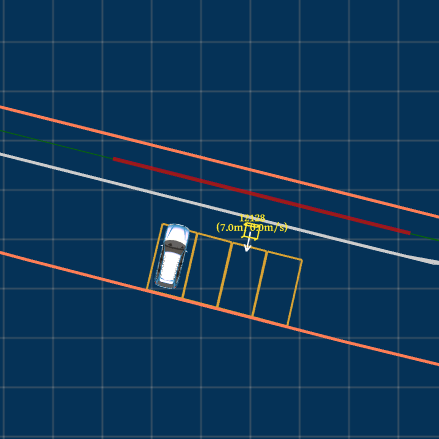}} 
\end{minipage}\hfill
\begin{minipage}{0.19\textwidth}
 \centering
 \subfloat[Vehicle Obstacle\label{figure:vehicle obstacle}]{\includegraphics[width=1.0\linewidth]{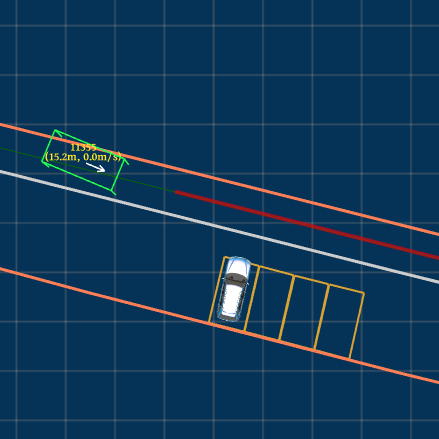}}
\end{minipage}\hfill
\begin{minipage}{0.19\textwidth}
 \centering
 \subfloat[Two Vehicle Obstacles\label{figure:two_vehicle_obstacles}]{\includegraphics[width=1.0\linewidth]{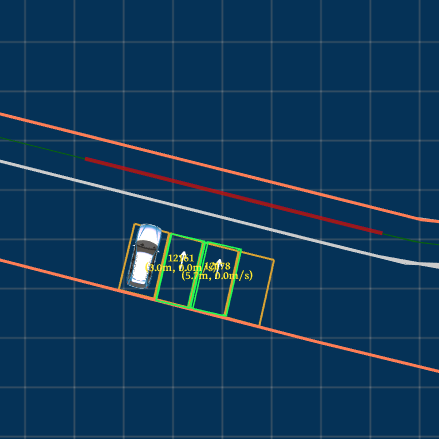}}
\end{minipage}\hfill
\begin{minipage}{0.19\textwidth}
 \centering
 \subfloat[Two Pedestrian Obstacles\label{figure:two_pedestrain obstacles}]{\includegraphics[width=1.0\linewidth]{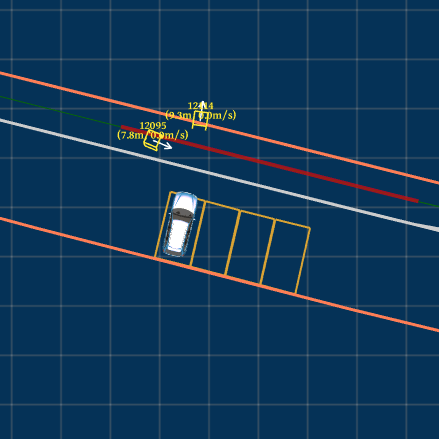}}
\end{minipage}
\begin{minipage}{0.19\textwidth}
 \centering
 \subfloat[Curved RoI Boundary\label{figure:curved_roi_boundary}]{\includegraphics[width=1.0\linewidth]{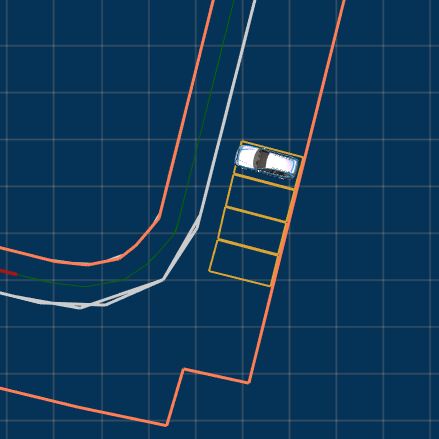}}
\end{minipage}
\caption{Settings for end-to-end valet parking simulation cases in Table~\ref{table:valet_parking_complex_simulation_data}}
\label{figure:complex_simulation_reuslts}
\end{figure*}


\begin{table}[h!]
\centering
\caption{Valet parking, computation time scales w.r.t. complex simulations in Fig.~\ref{figure:complex_simulation_reuslts}, where $N_{OA}$ is the number of Type A obstacles, i.e., RoI boundary line segments, and $N_{OB}$ is the number of Type B obstacles, i.e., vehicles' or pedestrians' boundary line segments. $t_{f,H}$ and $t_{f,TDR}$ are the mean time to generate each planning frame by H-OBCA and TDR-OBCA respectively. 
$t_{t,H}$ and $t_{t,TDR}$ are the total trajectory generation time by H-OBCA and TDR-OBCA respectively.
All the time are measured in second. 
}
\setlength{\tabcolsep}{0.5em}
\begin{tabular}{c|c|c|c|c|c}
\hline
 Case ID & (a) & (b) & (c) & (d) & (e) \\
 \hline
 $N_{OA}$ &5 &5 &6 &6&9\\
 \hline
 $N_{OB}$ & 4&4 &8 &8&0\\
 \hhline{=:=:=:=:=:=}
 $t_{f,H}$ (s)& N.A. &0.029& N.A.& 0.021& N.A. \\
 \hline
 $t_{f,TDR}$ (s) &0.017 &  0.016 & 0.022 & 0.019 & 0.015\\
 \hline
 \textbf{Improved Rate} & N.A. & \textbf{44.82\%} & N.A.& \textbf{9.52\%} & N.A. \\
  \hhline{=:=:=:=:=:=}
  $t_{t,H}$ (s) & N.A. & 1.80 & N.A. & 2.52 & N.A. \\
  \hline
   $t_{t,TDR}$ (s) & 1.08 & 1.74 &1.97 & 1.61 & 2.29 \\
   \hline
    \textbf{Improved Rate} & N.A. & \textbf{3.33}\% & N.A.& \textbf{36.11\%} & N.A. \\
 \hline
\end{tabular}
\label{table:valet_parking_complex_simulation_data}
\end{table}

Fig.~\ref{figure:complex_simulation_reuslts} and Table~\ref{table:valet_parking_complex_simulation_data} show TDR-OBCA's robustness and efficiency in two aspects, i) TDR-OBCA is able to handle cases where H-OBCA fails; 2) the TDR-OBCA's computation time is less than H-OBCA.
Moreover, figures in Fig.~\ref{figure:simulation_reuslts} show TDR-OBCA's robustness in different application scenarios, such as valet parking, pulling over and hailing.

\subsubsection{Simulation Summary}
\label{sec:simulation_summary}
TDR-OBCA is robust compared to competitive algorithms.
Together with the regular planner, it is able to handle different scenarios with various obstacles, which is essential for autonomous driving. Furthermore, in the next subsection with real road test results, we show the trajectories generated by TDR-OBCA meet control-level smoothness for autonomous driving.

\subsection{Real world road tests}
\label{sec:road_experiment}
We have integrated TDR-OBCA with the planner module in Apollo Open-Source Autonomous Driving Platform and validated its robustness and efficiency with hundreds of hours of road tests in both USA and China.
Fig.~\ref{figure:open_space_architect_1} shows the architecture of the trajectory provider with highlighted TDR-OBCA algorithm module.
The trajectory provider takes inputs from map, localization, routing, perception and prediction modules to formulate the Region of Interest and transfers obstacles into line segments.
The trajectories generated by TDR-OBCA are passed into three post-processing modules before being sent to the vehicle's control layer.
In trajectory stitching module, each trajectory is stitched with respect to the ego vehicle's position.
The safety check module guarantees the ego vehicle has no collisions to moving obstacles.
The trajectory partition module divides the trajectory based on its gear location. 
Finally, the trajectory is sent to the control module, which communicates with Controller Area Network (CAN bus) module to drive the ego vehicle based on this provided trajectory.

\begin{figure}[]
\centering
\includegraphics[width=0.95\linewidth]{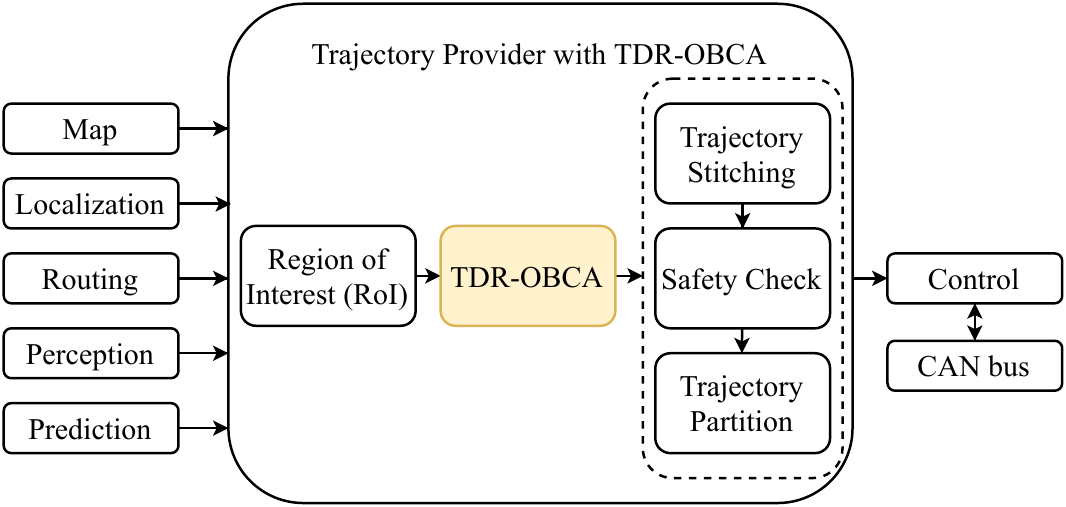}
\caption{TDR-OBCA application in Apollo Autonomous Driving Platform.}
\label{figure:open_space_architect_1}
\end{figure}

The three types of simulation scenarios presented in Section~\ref{sec:simulation_apollo} are all selected from hundreds of hours of real road tests.
For all their related road tests, the lateral control accuracy is high and in the range from 0.01 to 0.2 meters.

In Apollo planning module, we use a hybrid planer algorithm, which combines the free-space trajectory provider with the regular trajectory provider. TDR-OBCA is applied when the starting point or the end point of a trajectory is off driving road. Together with the regular driving planner, TDR-OBCA aims to maneuver the ego vehicle to a designated end position. It is usually applied but not limited to low speed scenarios. The road-test logic is shown in Algorithm~\ref{algorithm:parking}.

\begin{algorithm}
\SetAlgoLined
\KwResult{Maneuver the vehicle to from starting position to end position}
Generate RoI from map, routing and localization modules\;
Get obstacle shapes and locations from Perception and prediction modules\;
 \While{Scenario is not completed}{
 Check ego vehicle current status\;
 \If{At end pose}{
 Break \;
 }
  \eIf{Vehicle is on driving road}{
    Use on-road planner\;
}{Use TDR-OBCA, adjust position and maneuver forward or backward\;
Stitch trajectory making it start from the current vehicle position\;
Check collisions with moving obstacles\;
Divide trajectory according to gear positions\; }
   Generate control level smooth trajectory with speed profile\;
   Generate trajectory tracking control commands with control module\;
   Move vehicle with CAN bus module\;
   Update current CAN bus status for control module;
  }
\caption{Hybrid Planer with TDR-OBCA}
\label{algorithm:parking}
\end{algorithm}

Here we only take the valet parking scenario as an example to show TDR-OBCA's road-test performance in detail.
Fig.~\ref{figure:Valet_Parking} shows the results of three different valet parking experiments.

According to the ego vehicle's status, such as the rear-wheels-to-parking-spot distance and the heading, the car with our algorithm is able to generate either direct parking trajectories or zig-zag trajectories respectively.
Table~\ref{table:valet_parking_experimental_data} shows control performance (including lateral errors and heading angle errors at the end pose of each trajectory) of these three scenarios, where the controller of autonomous vehicle follows the planned trajectory generated by TDR-OBCA.
Our results confirm that the trajectories generated by TDR-OBCA are smooth, constrained by vehicle's kinodynamics and lead to low tracking errors in real road tests.

\begin{figure}[!h]
\begin{minipage}{0.16\textwidth}
 \centering
 \subfloat[stage 1: backward\label{figure:Valet_Parking_Reverse_1}]{\includegraphics[width=.9\linewidth]{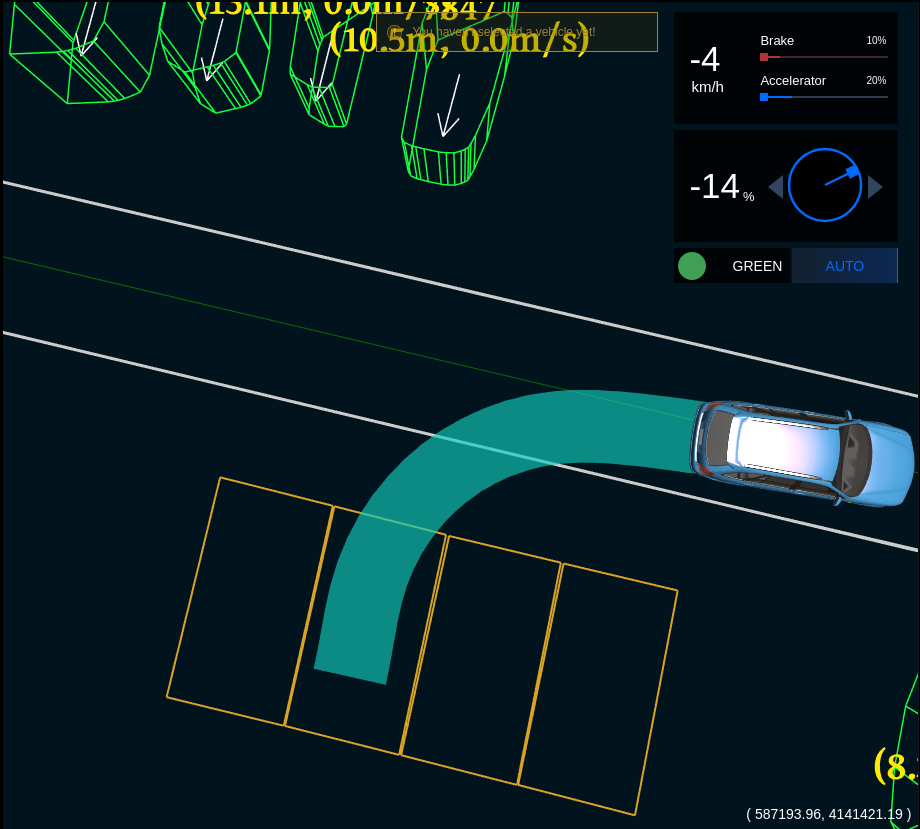}}
\end{minipage}\hfill
\begin {minipage}{0.16\textwidth}
 \centering
 \subfloat[stage 2: finishing\label{figure:Valet_Parking_Reverse_3}]{\includegraphics[width=.9\linewidth]{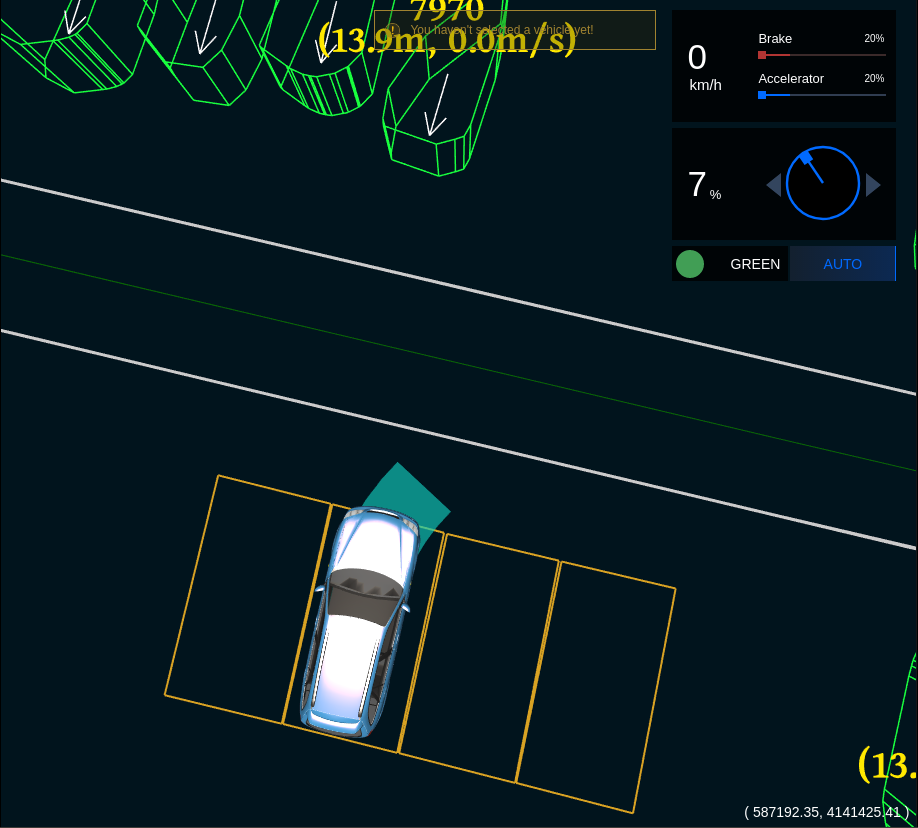}}
\end{minipage}\hfill
\begin{minipage}{0.16\textwidth}
 \centering
 \subfloat[whole trajectory\label{figure:Valet_Parking_Reverse_Trajectory_Optimizer_Visualization_2}]{\includegraphics[width=.9\linewidth]{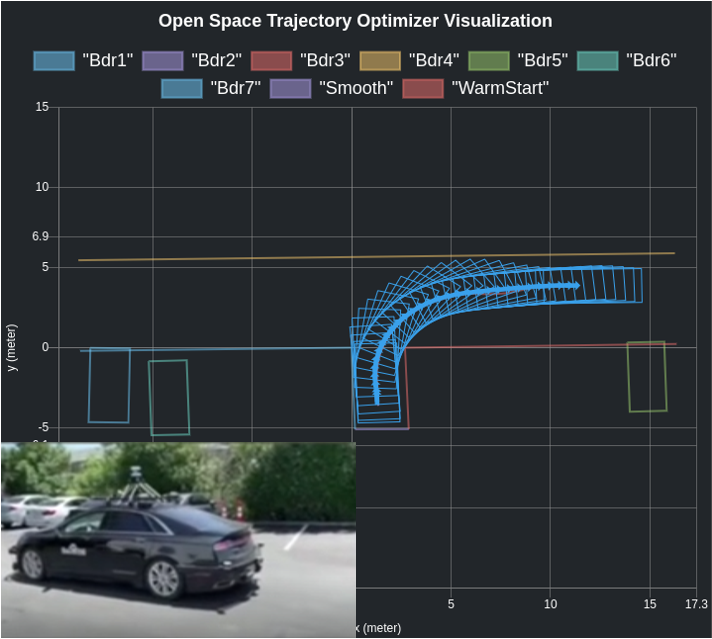}}
\end{minipage}\hfill
\begin{minipage}{0.16\textwidth}
 \centering
 \subfloat[stage 1: forward\label{figure:Valet_Parking_Short_Reverse_1}]{\includegraphics[width=.9\linewidth]{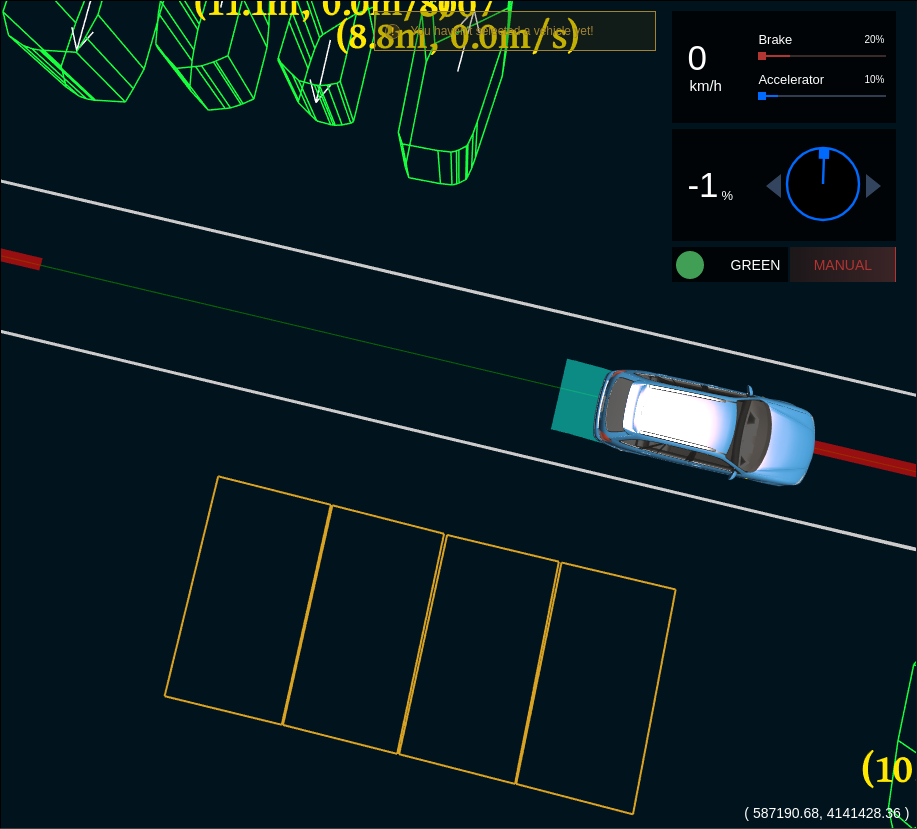}}
\end{minipage}\hfill
\begin {minipage}{0.16\textwidth}
 \centering
 \subfloat[stage 2: backward\label{figure:Valet_Parking_Short_Reverse_2}]{\includegraphics[width=.9\linewidth]{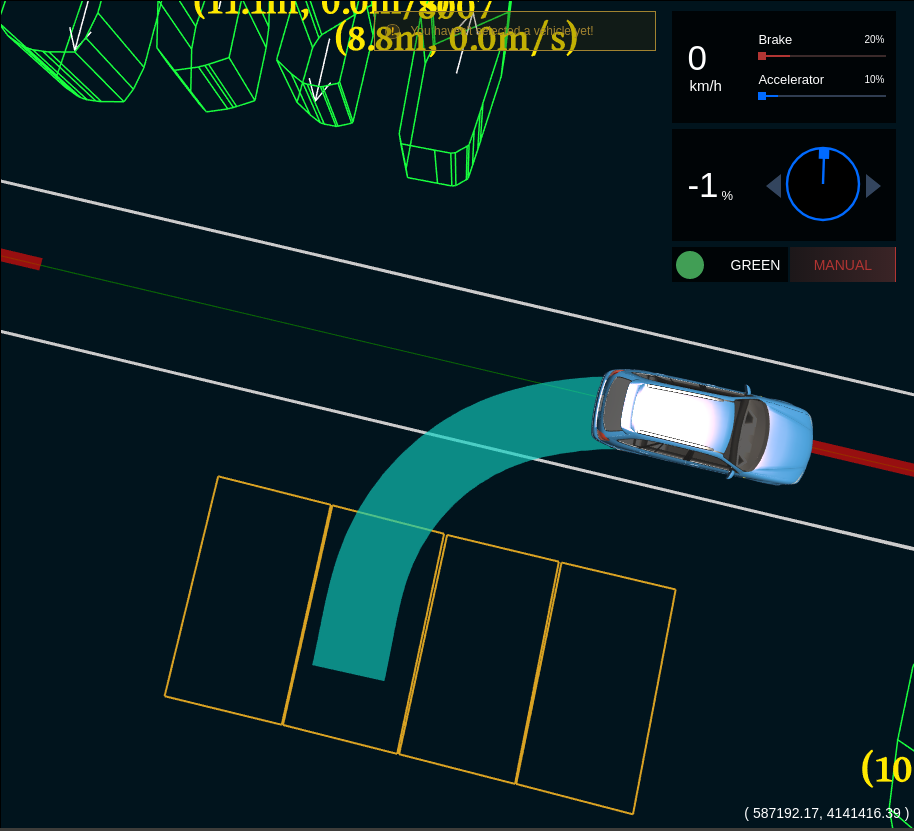}}
\end{minipage}\hfill
\begin{minipage}{0.16\textwidth}
 \centering
 \subfloat[whole trajectory\label{figure:Valet_Parking_Short_Reverse_Trajectory_Optimizer_Visualization_2}]{\includegraphics[width=.9\linewidth]{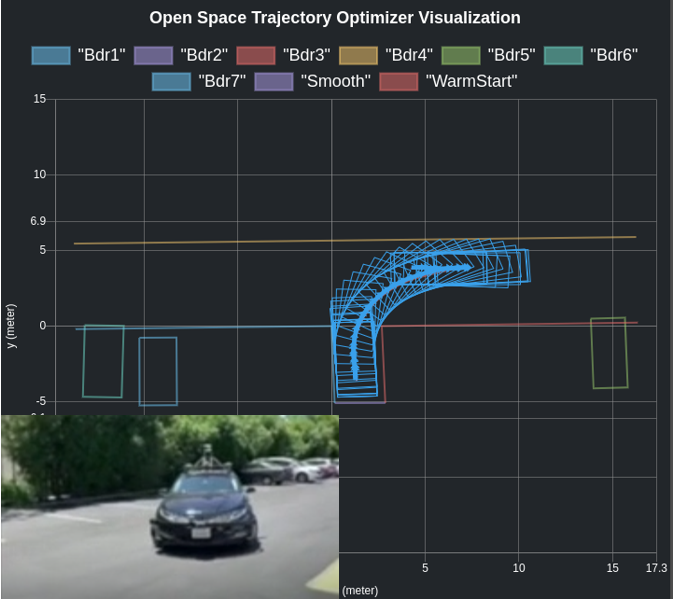}}
 \end{minipage}\hfill
\begin{minipage}{0.16\textwidth}
 \centering
 \subfloat[stage 1: forward\label{figure:Valet_Parking_Zigzag_1}]{\includegraphics[width=.9\linewidth]{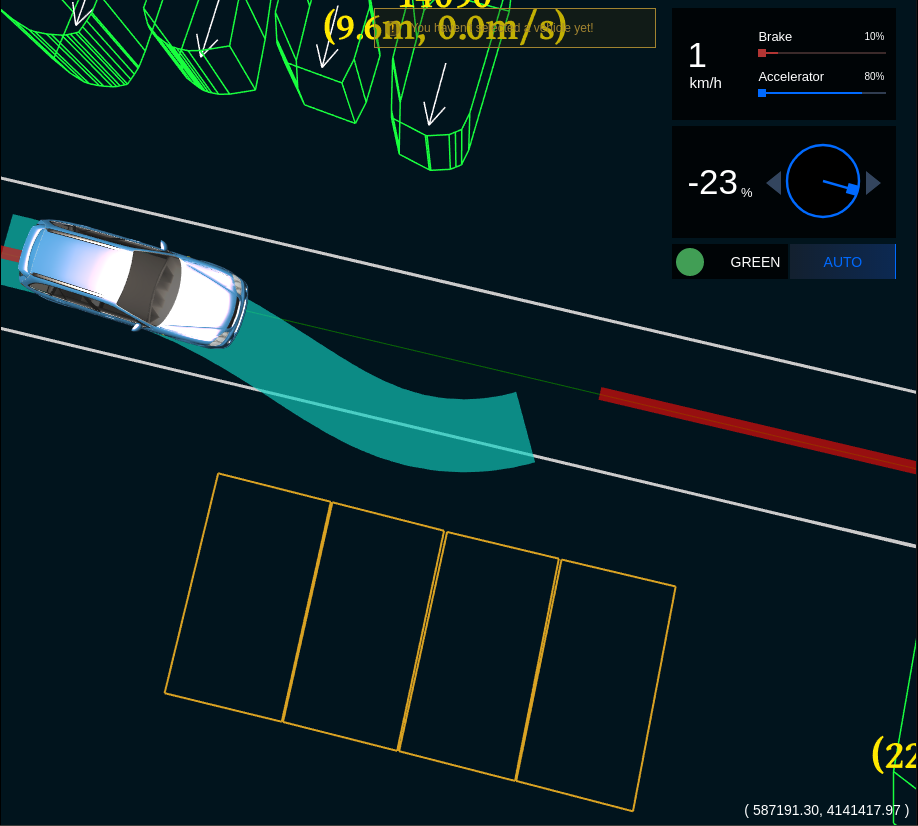}}
\end{minipage}\hfill
\begin{minipage}{0.16\textwidth}
 \centering
 \subfloat[stage 2: backward\label{figure:Valet_Parking_Zigzag_4}]{\includegraphics[width=.9\linewidth]{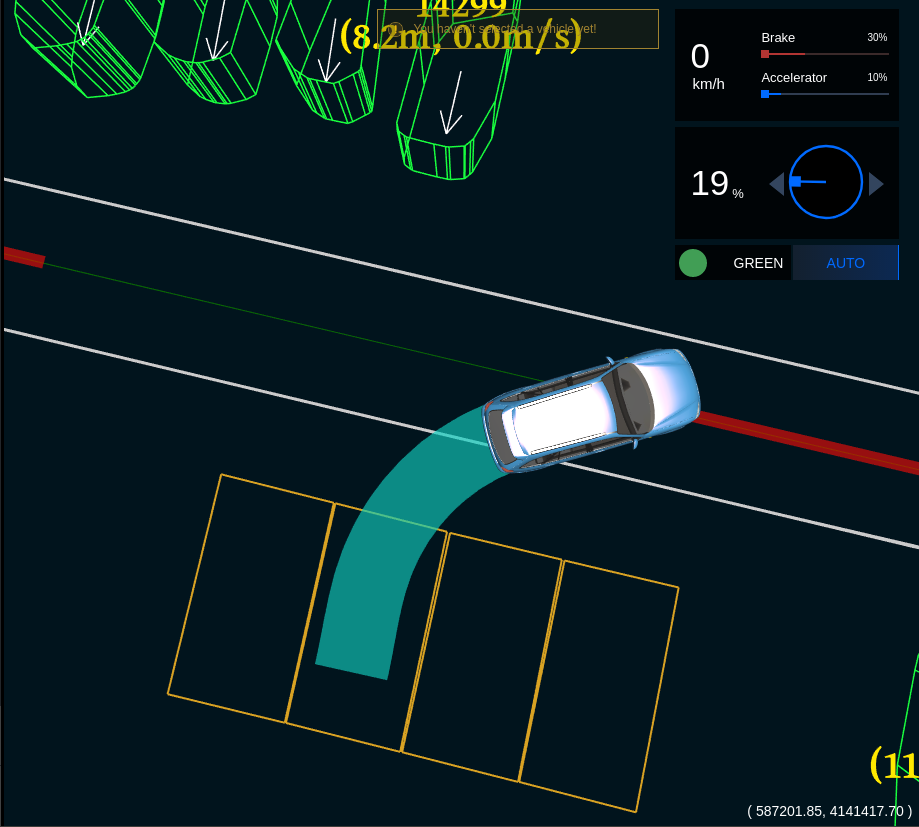}} 
\end{minipage}\hfill
\begin{minipage}{0.16\textwidth}
 \centering
 \subfloat[whole trajectory\label{figure:Valet_Parking_Zigzag_Trajectory_Optimizer_Visualization_2}]{\includegraphics[width=.9\linewidth]{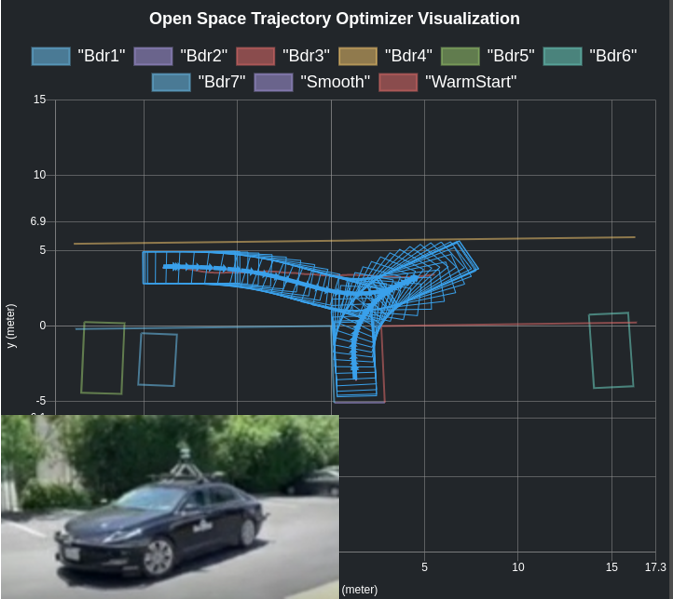}} 
 \end{minipage}
\caption{Real road test planning trajectories including scenarios: 1) reverse parking:~\ref{figure:Valet_Parking_Reverse_1},~\ref{figure:Valet_Parking_Reverse_3} and~\ref{figure:Valet_Parking_Reverse_Trajectory_Optimizer_Visualization_2}; 2) short reverse parking:~\ref{figure:Valet_Parking_Short_Reverse_1},~\ref{figure:Valet_Parking_Short_Reverse_2} and~\ref{figure:Valet_Parking_Short_Reverse_Trajectory_Optimizer_Visualization_2}; 3) zig-zag parking:~\ref{figure:Valet_Parking_Zigzag_1},~\ref{figure:Valet_Parking_Zigzag_4} and~\ref{figure:Valet_Parking_Zigzag_Trajectory_Optimizer_Visualization_2}.}
\label{figure:Valet_Parking}
\end{figure}

\begin{table}[h!]
\centering
\caption{Valet parking: trajectory end pose accuracy on real road tests}
\setlength{\tabcolsep}{0.5em}
\begin{tabular}{ p{3.7cm} | p{2.0cm} p{2.2cm} }
 \hline
 Parking Scenarios (Partitions) & Lateral Error(m)  & Heading Error(deg)\\ 
 \hhline{=:==}
 Reverse Parking (test run 1) & 0.0337 & 0.042\\
 \hline
 Reverse Parking (test run 2) & 0.0374 & 0.068 \\
 \hline
 Reverse Parking (test run 3) & 0.0446 & 0.098\\
 \hline
 Short Reverse Parking & 0.0645 & 0.137\\
 \hline
 Zigzag Parking (forward part)& 0.0331 & 1.271\\
 \hline
 Zigzag Parking (backward part)& 0.0476 & 0.037\\
 \hline
\end{tabular}
\label{table:valet_parking_experimental_data}
\end{table}

\section{Conclusion}
In this paper, we present TDR-OBCA, a robust, efficient and control friendly trajectory generation algorithm for autonomous driving in free space.
In TDR-OBCA, two new warm start methods and the optimization problem's reformulation dramatically decrease the simulation failure rate by $97\%$ and computation time by up to $44.82\%$, and increase driving comfort by reducing the steering control output by more than $13.53 \%$, thus making it more feasible as a real world application.. 

The results have been tested on various scenarios and hundreds of hours' road tests. In the future, we will focus on further improving computational efficiency and considering more driving comfort aspects to different scenarios~\cite{kilinc2012determination}~\cite{bae2020self}.

\bibliographystyle{IEEEtran}
\bibliography{IEEEabrv,./refs}

\end{document}